\documentclass[submission,copyright,creativecommons]{eptcs}
\providecommand{\event}{ICLP 2023} 
\usepackage{iftex}

\ifpdf
  \usepackage{underscore}         
  \usepackage[T1]{fontenc}        
\else
  \usepackage{breakurl}           
\fi

\usepackage{amsfonts}
\usepackage{amssymb}
\usepackage{amsthm}
\usepackage{amsmath}
\usepackage{todonotes}

\newtheorem{example}{Example}

\newtheorem{definition}{Definition}

\newtheorem{proposition}{Proposition}

\usepackage{xcolor}
\usepackage[normalem]{ulem}

\def\AAF{AAF}
\def\Args{Args}
\def\Att{Att}
\def\arg{\alpha}

\def\applies{\Vdash}
\def\PS{PS}
\def\Worlds{\mathcal{W}}
\def\GE{G}
\def\XE{G}

\newcommand{\abaf}{\ensuremath{\langle {\cal L}, \, {\cal R}, \, {\cal A},\, \overline{ \vrule height 5pt depth 3.5pt width 0pt \hskip0.5em\kern0.4em}\rangle}}
\def\contrary{\overline{ \vrule height 5pt depth 3.5pt width 0pt \hskip0.5em\kern0.4em}}

\newcommand{\lang}{\ensuremath{\mathcal{L}}}
\newcommand{\asms}{\ensuremath{\mathcal{A}}}
\newcommand{\rules}{\ensuremath{\mathcal{R}}}
\newcommand{\facts}{\ensuremath{\mathcal{F}}}
\newcommand{\some}{\ensuremath{\chi}}
\newcommand{\argur}[3]{#1 \vdash_{#3} #2} 
\newcommand{\ruleset}{\ensuremath{\mathcal{S}}}
\newcommand{\sent}{\ensuremath{s}}

\newcommand{\asm}{\ensuremath{a}}
\newcommand{\asmset}{\ensuremath{A}}

\newcommand{\abafT}{\ensuremath{\langle {\cal L}^T, \, \mathcal{R}^T, \, \mathcal{A}^{T},\, \overline{ \vrule height 5pt depth 3.5pt width 0pt \hskip0.5em\kern0.4em}^{T}\rangle}}

\title{Understanding ProbLog as Probabilistic Argumentation}
\author{Francesca Toni \qquad Nico Potyka
\institute{Department of Computing\\ Imperial College London, UK}
\email{\{ft, n.potyka\}@ic.ac.uk}
\and
Markus Ulbricht  
\institute{Department of Computer Science \\
Leipzig University, Germany}
\email{\quad mulbricht@informatik.uni-leipzig.de}
\and 
Pietro Totis
\institute{Department of Computer Science\\ KU Leuven, Belgium}
\email{\quad pietro.totis@kuleuven.be}
}
\def\titlerunning{ProbLog as Probabilistic Argumentation}
\def\authorrunning{ Toni, Potyka, Ulbricht \& Totis}
\begin{document}
\maketitle

\begin{abstract}
ProbLog is a popular probabilistic logic programming 
language/tool, widely used for applications requiring to deal with inherent uncertainties in structured domains. In this paper we study 
connections between ProbLog and a variant of another well-known formalism combining symbolic reasoning and reasoning under uncertainty, 
i.e. probabilistic argumentation. Specifically, we show that ProbLog is an instance of a form of Probabilistic Abstract Argumentation (PAA) 
that builds upon  Assumption-Based Argumentation (ABA)
.    
The connections pave the way towards equipping ProbLog with 
alternative semantics, inherited from PAA/PABA, as well as obtaining novel argumentation semantics for PAA/PABA, leveraging on 
prior
connections between ProbLog and argumentation. 
Further, the connections pave the way towards novel forms of argumentative
explanations for ProbLog's outputs.
\end{abstract}

\section{Introduction}

ProbLog \cite{ProbLog,ProbLog15} is a popular probabilistic logic programming formalism, 
equipped with efficient tools in support of applications\footnote{See \url{https://dtai.cs.kuleuven.be/problog/}.}
In a nutshell, ProbLog programs amount to logic programs where facts may be equipped with probabilities. 
Thus, ProbLog can be naturally
used for applications requiring to deal with inherent uncertainties, e.g. in neuro-symbolic settings to learn how to operate on images~\cite{deepProbLog18}, object tracking~\cite{ObjTraciking}, modelling metabolic networks\cite{MetaNetworks}, or synthesising inductive data models~\cite{synth}.

As a knowledge representation and reasoning (KRR) formalism, ProbLog can be seen as bringing together symbolic reasoning
and reasoning under uncertainty. Thus, it makes sense to wonder how it relates to the family of KRR formalisms broadly referred to as \emph{Probabilistic Argumentation} (PA) \cite{hunter2021probabilistic}, which also bring together these two types of reasoning. Indeed, some existing works already study the connections between the two, in particular 
in the presence of inherent uncertainties, e.g. in neuro-symbolic settings 
with images \cite{deepProbLog18}.
Also, \textsc{sm}ProbLog~\cite{SMProbLog}
extends ProbLog beyond the 
standard well-founded model semantics~\cite{WFM} towards the stable model semantics~\cite{SMS} so as to capture a form of 
 PA. 
Further, 
\cite{Bistarelli20}  shows how the form of 
PA 
in \cite{nir}
can be implemented in ProbLog.
These existing works thus focus on  
showing how (variants of) Problog can 
obtain (forms of) 
PA. 
In this paper, instead, we focus on the opposite direction, specifically on whether 
(some forms of) 
PA can be used to obtain 
(variants of) ProbLog
.
In summary, our contribution is two-fold:
\begin{itemize}
\item We define a new instance of 
the well-known Assumption-based Argumentation (ABA)~\cite{ABA,ABAhandbook} 
and  use it to instantiate Probabilistic Abstract Argumentation (PAA) \cite{PAA-PABA} 
(Section~\ref{sec:newABA}); 
\item We 
formally relate 
ProbLog and the proposed instance of PAA  (Section~\ref{sec:results}).
\end{itemize}
 This reinterpretation of ProbLog in argumentation terms opens new avenues (as discussed in Section~\ref{sec:concl})
 .  

\section{Background}

\paragraph{Logic Programs.} A \emph{logic program} (LP) is a set of (implicitly universally quantified) rules of the form
$ l_0 \leftarrow l_1, \ldots, l_m$, 
with $m \geq 0$, $l_0$ an atom (called the \emph{head} of the rule) and each $l_i$ (for $1 \leq i\leq m$) an atom or the \emph{negation-as-failure} $not \, a_i$ of an atom $a_i$.
If $m=0$ then the rule is called a \emph{fact}.
The \emph{Herbrand base} of an LP is the set of all ground atoms obtained from the predicate
, function/constant symbols in the LP.  

\paragraph{ProbLog.}
The material in this section is adapted from \cite{ProbLog,ProbLog15}.
A \emph{ProbLog program}
is a set $R$ of rules  as in an LP together with a set $F$ of \emph{probabilistic facts} of the form 
$ p:: l_0 \leftarrow$, where $ p \in [0,1]$ and $ l_0$ is an atom  that is  not the head of any rule in $R$.
The semantics of a ProbLog program is based on distribution semantics~\cite{distribution}:
a ProbLog program $T= F\cup R$ defines a probability distribution over 
LPs $L = F'\cup R, F'\subseteq \{l\leftarrow | p::l\leftarrow \in  F\}$:
$$ P(L|T)= \prod_{l
\leftarrow \in F', p
::l
\leftarrow \in F} p
\cdot \prod_{l
\leftarrow 
\not\in F', p
::l
\leftarrow \in F} (1-p
)$$
Then, for a \emph{query} $q$
,  the \emph{success probability} for $q$, given $T$, is\footnote{A query $q$ is a possibly non-ground atom from the Herbrand base of $T$ and $\theta$ denotes a possibly empty variable substitution. }

$$P_s(q|T) = \sum_{L=F'\cup R,  
F'\subseteq \{l\leftarrow | p::l\leftarrow \in  F\},\exists \theta :L\models q\theta} P(L|T)$$
for $\models$ the chosen logic programming semantics: if $L
$ is a \emph{positive LP} (without negation as failure), as in \cite{distribution}, then this is the least Herbrand model
; otherwise, let it be the well-founded model
~\cite{WFM}.

\begin{example}
\label{ex:ProbLog}
Consider the (propositional
) LP amounting to the following rules: 

\hspace*{3.5cm} $a\leftarrow b,not \, c;
\quad \quad
b\leftarrow; \quad \quad
d \leftarrow not \,d$
\\
(referred to in short as $\rho_1, \rho_2, \rho_3$, respectively).
Then, for $R=\{\rho_1, \rho_3\}$ and
$F=\{0.3:: b\leftarrow\}$, $T=F\cup R$ is a ProbLog program.
Given that $R\not \models a$ (as the  well-founded model of $R$ deems $a,b,c$ false and $d$ undecided)
and
$R\cup\{\rho_2\} \models a$ (as the  well-founded model of $R\cup\{\rho_2\}$ deems $a,b$ true, $c$ false and $d$ undecided),
$P_s(a | T)=P(R\cup\{\rho_2\} |T)=0.3$.
\end{example}
ProbLog typically assumes that
the well-founded model of every LP 
$L \!\subseteq \!F' \!\cup \!R, F'\!\subseteq \! \{l\leftarrow \!| p::l\leftarrow \!\in\!  F\}$  is 
two-valued~\cite{ProbLog,SMProbLog}, but 
the well-founded model may be three-valued~\cite{prism3val}
, as in the example 
.

\paragraph{Abstract Argumentation.}
An Abstract Argumentation (AA) framework~\cite{Dung95} is a pair $ (\Args, \Att)$  with $\Args$ the \emph{arguments} and $\Att \!\subseteq \!\Args\!\times\!\Args$ the \emph{attack} relation. AA frameworks can be equipped with various argumentative semantics~\cite{Dung95}, e.g. the \emph{grounded} or \emph{stable extensions} semantics~\cite{Dung95}, identifying sets of ``acceptable arguments''. For illustration,
given $\AAF=  (\Args, \Att)$ with $\Args=\{\alpha, \beta\}$ and $\Att=\{(\beta,\beta)\}$, the grounded extension of $\AAF$ is $\{\alpha\}$ and there is no stable extension of $\AAF$.

\paragraph{Assumption-based argumentation (ABA).} 

An
{\em ABA framework} (as originally proposed in \cite{ABA}, but presented here following more recent accounts by \cite{
ABAtutorial,
ABAhandbook}) is a tuple \abaf{}
where

\begin{itemize}
\item 
$\langle \lang, \rules\rangle$ is a deductive system,
 where $\lang$ is a \emph{language} and $\rules$ is a set of
 \emph{(inference) rules} of the form $\sent_0 \leftarrow \sent_1,\ldots, \sent_m $ ($m \ge 0, \sent_i \in \lang$, for $1\leq i \leq m$); 

\item 
$\asms$ $\subseteq$ $\lang$ is a (non-empty) set
of {\em assumptions}; 

\item 
$\contrary$ is a total mapping from $\asms$ into
 $\lang$, where $\overline{\asm}$ is the {\em contrary} of $\asm$, for $\asm
  \in \asms$.
\end{itemize}
Given a rule $\sent_0 \gets \sent_1, \ldots,
\sent_m$, $\sent_0$ is referred to as the {\em head} 
and $\sent_1,\ldots, \sent_m$ as the {\em body}; 
if $m=0$ then the  body is said to be {\em empty}.
{\em Flat ABA frameworks} are restricted so that assumptions are not heads 
of rules. 

Differently from AA, where arguments and attacks are given, in flat ABA they are derived from the building blocks of ABA frameworks. Specifically,  {\em arguments} are deductions of claims using rules and 
supported by assumptions, and {\em attacks} are directed at the
assumptions in the support of arguments: 
\begin{itemize}
\item 
\emph{an argument for (the claim) $\sent \in \lang$ 
supported by $\asmset \subseteq \asms$ and $\ruleset \subseteq \rules$
}
(denoted $\argur{\asmset}{\sent}{\ruleset}$
) is a finite tree with nodes labelled by
sentences in $\lang$ or by $true$\footnote{Here $true$ stands for the empty body of rules.}, the root labelled by $\sent$, leaves either $true$ or
assumptions in $\asmset$, and non-leaves $\sent'$ with, as children,
the elements of the body of one rule in \ruleset{} 
with head $\sent'$ ($true$ in the case of a rule with an empty body), and \ruleset{} the set of all these rules;
\item 
 an argument 
$\argur{\asmset_1}{\sent_1}{\ruleset_1}$ 
{\em attacks} an argument
$\argur{\asmset_2}{\sent_2}{\ruleset_2}$  
iff
$\sent_1=\overline{\asm}$ for some
$\asm \in \asmset_2$.
\end{itemize}
%
%
Given an ABA framework \abaf, let $\Args$ be the set of all arguments
and $\Att$ be defined as above.
Then $(\Args,\Att)$ is an AA framework (that we refer to as \emph{the AA framework for \abaf}) and standard argumentative semantics 
for the latter can be used to determine, e.g.,  grounded 
extensions.\footnote{ABA semantics were originally defined in terms of sets of assumptions and attacks between them \cite{ABA}, but can be reformulated, for flat ABA frameworks, 
in terms of sets of arguments and attacks between them (see~\cite{ABAtutorial}), as given here. }

\paragraph{LPs as Flat ABA Frameworks.}
Flat ABA can be instantiated to capture logic programming under several semantics~\cite{ABA}, including the well-founded model~\cite{WFM} and the stable model semantics~\cite{SMS}.  

\begin{example}
\label{ex:LP-ABA}
The semantics of the LP
 $\{\rho_1, \rho_2, \rho_3\}$ from Example~\ref{ex:ProbLog}
can be captured by understanding it as a flat ABA framework  \abaf\ with
\begin{itemize}
    \item $\rules =\{\rho_1, \rho_2,\rho_3\}$;  
    \item $\lang=\{a,b,c,d,not\, a, not\, b, not\,c,not\,d\}$ (this is the Herbrand base of the LP together with all the negation as failure literals that can be constructed from its Herbrand base);
    \item $\asms=\{not\, a, not\, b, not\,c,not\,d\}$ (this is the set of all negation as failure literals in $\lang$);
    \item $\overline{not\,X}=X$ for all $X \in \{a,b,c,d\}$.
\end{itemize}
$(\Args,\Att)$ can then be obtained from this ABA framework, with:
\begin{itemize}
    \item $\Args=\{
    \argur{\{not\,c\}}{a}{\{\rho_1,\rho_2\}},\; \argur{\{\}}{b}{\{\rho_2\}}, \; \argur{\{not\, d\}}{d}{\{\rho_3\}}\} \cup 
    \{
    \argur{\{not \, X\}}{not \, X}{\{\}}| X \in \{a,b,c,d\}\}$
    \item 
        $\Att=  \{(\argur{\{not\,c\}}{a}{\{\rho_1,\rho_2\}},\argur{\{not \, a\}}{not \, a}{\{\}}),$ 
        \quad $(\argur{\{not\,d\}}{d}{\{\rho_3\}},\argur{\{not \, d\}}{not \, d}{\{\}}),$ \\ \hspace*{1.05cm} $(\argur{\{not\,d\}}{d}{\{\rho_3\}},\argur{\{not \, d\}}{d} {\{\rho_3\}}),$ 
        \quad  $ (\argur{\{\}}{b}{\{\rho_2\}},\argur{\{not \, b\}}{not \, b}{\{\}}) \}.$
    \end{itemize}
$\rules$, as a LP, admits no stable model~\cite{SMS}, and there is no stable extension~\cite{Dung95} of this ABA framework.
The well-founded model~\cite{WFM} of $\rules$ deems $a,b$ true, $c$ false and $d$ undecided, and 
the grounded extension ~\cite{Dung95} of \abaf\ accepts 

\hspace*{3cm} $\{\argur{\{not\,c\}}{a}{\{\rho_1,\rho_2\}},\; \argur{\{\}}{b}{\{\rho_2\}}, \; 
    \argur{\{not \, c\}}{not \, c}{\{\}}\}.$
    \\
The two correspond in the sense that the atomic claims of accepted arguments in the grounded extension are true in the well-founded model, the negation as failure claims of accepted arguments are the complement of false atoms in the well-founded model, and no other atoms are true or false therein (and vice versa, i.e. every true atom and the 
negation as failure of every false atom in the well-founded model are claims of accepted arguments in the grounded extension, and no other arguments exist therein). Our results will leverage on the correspondence illustrated here (see~\cite{ABA} for formal results, including correspondences between other semantics of LPs and notions of extensions in flat ABA~\cite{ABA}).
\end{example}

\paragraph{Probabilistic Abstract Argumentation (PAA).} 

The material in this section is adapted from \cite{PAA-PABA} (focusing on a single juror).
A \emph{PAA framework} basically consists of an AA framework 
and a probability distribution assigning probabilities to subframeworks. Formally, 
it is a triple ${(\AAF,\PS,\applies)}$ where:
\begin{itemize}
\item $ \AAF \!=\! (\Args, \Att)$ is an AA framework;

\item $\PS \!=\! (\Worlds,P)$ is a probability space, with $\Worlds$ the 
\emph{world set} and $P$ a \emph{probability distribution} over $\Worlds$;

\item $\applies \,\subseteq \Worlds \times \Args$ 
specifies  which arguments are {\em applicable} in  worlds in $\PS$,
where $\arg \in \Args$ is applicable in $w \in \Worlds$ iff $(w,\arg) \in \applies$ (we also write $w\applies \arg$ for $(w,\arg) \in \applies$). 
\end{itemize}
Given a possible world $w \in \Worlds$,
the {\em AAF wrt $w$}, denoted by $\AAF_{w} = (\Args_{w}, \Att_{w})$, is defined by restricting $\AAF$ to the arguments applicable in $w$. More formally,
$\Args_w =  \{ \arg \in \Args \,\,|\,\, w \applies \arg\,\}$ and $\Att_w = \Att\, \cap (\Args_w \times \Args_w).$
We denote the {grounded extension} 
of $\AAF_w$ by $\GE(\AAF_w)$.
Then, the {\em grounded 
probability of argument} 
$\arg \in \Args$ 
is defined as follows
:

$$Prob_{\XE}(\arg) = \sum \limits_{w \in \Worlds: \arg \in \XE(\AAF_w) } P(w).$$
Intuitively, 
$Prob_{\XE}(\arg)$ is the probability that $\arg$ 
is accepted in a possible world (under the grounded extension semantics).
PAA is a form of probabilistic argumentation under the constellation approach \cite{hunter2021probabilistic}.

\section{ProbLog-ABA and ProbLog-PAA}
\label{sec:newABA}
Several instances of (flat and non-flat) ABA have been studied \cite{ABA
,Quentin}
(including the flat logic programming instance illustrated in Example~\ref{ex:LP-ABA}~\cite{ABA}). Here, we introduce the following (flat) instance, referred to as 
\emph{ProbLog-ABA} since, as we will show, it is a first step towards capturing ProbLog programs as PAA.

\begin{definition}
\label{def:problogaba}
A \emph{ProbLog-ABA framework}
is a flat ABA framework \abaf\ where
\begin{itemize}
\item $\rules$ is a (ground)\footnote{This is in line with standard practices when studying the semantics of LPs: they are assumed to stand for the set of all their ground instances over their Herbrand universe.}  LP;
    \item $\lang$ is a set of atoms and negation as failure literals; formally, for $HB$ the Herbrand Base of $\rules$ and $HB^{not}=\{not\,p | p \in HB\}$,
    $\lang=HB \cup HB^{not}\cup\{\some\}$, for a special symbol $\some\not\in HB \cup HB^{not}$;
    \item $\asms= HB^{not} \cup \facts $, for some $\facts \subseteq HB$;
    \item for any $not\,p \in \asms$, $\overline{not\,p}=p$; for any $\phi \in \facts$, $\overline{\phi}=\some$.
\end{itemize}
\end{definition}
Basically, ProbLog-ABA frameworks generalise the ABA frameworks capturing LPs by including additional assumptions ($\facts$), all having the same contrary ($\some$) not occurring in the LP ($\rules$). 
These additional assumptions can be seen as playing the role of abducibles in abductive logic programming~\cite{abduction}.
They amount to the atoms in probabilistic facts in ProbLog programs $T=R \cup F$, as follows.\footnote{ Given that we are focusing on semantics, in the remainder we assume that the ProbLog program is ground.}  

\begin{definition}
The \emph{ProbLog-ABA framework corresponding to $T$} is $\abaf$ as in Definition~\ref{def:problogaba} where $\rules=R$ and $\facts=\{l |(p :: l\leftarrow) \in F\}$.
\end{definition}

We can then instantiate  PAA  with ProbLog programs, so that the AA framework component is drawn from the ProbLog-ABA framework corresponding to a given ProbLog program
and the probability space is defined to mirror the probability distribution over LPs captured by the ProbLog program, as follows.

\begin{definition}
\label{ProbLog-PAA}
Let
$\abafT$ be the ProbLog-ABA framework corresponding to $T$, for $\asms^T=HB^{not} \cup \facts$, and let $\AAF^T=(\Args^T,\Att^T)$ be the AA framework for $\abafT$.
Then, the \emph{PAA framework corresponding to $T$} is $(\AAF^T, \PS^T,  \applies^T)$ where $\PS^T=(\Worlds^T,P)$ is the probability space with

 $\bullet$ $\Worlds^T=2^{\facts}$,
 
    $\bullet$ 
    for $w \in \Worlds^T$:
    $ P(w)= \prod_{l \in w, p::l\leftarrow \in F} p \cdot \prod_{l \in \facts \setminus w, p::l\leftarrow \in F} (1-p),$
\\
and $\applies^T$ is such that, for $w \in \Worlds^T$, $\arg \in \Args^T$, $\arg=\argur{\asmset}{\sent}{\ruleset}$:
\quad 
$w \applies^T \arg \mbox{ iff} \asmset \cap \facts \subseteq w$.
\end{definition}
In line with \cite{PAA-PABA}, we assume that 
the AA frameworks in PAA frameworks corresponding to ProbLog programs are finite.
Given 
such a PAA framework, we can measure the grounded probability of arguments (as in standard PAA), as well as the  \emph{grounded probability of queries}, as follows.

\begin{definition}
\label{def:grounded prob of queries}
Let $T$ be a (ground) ProbLog program and 
$(\AAF^T, \PS^T,  \applies^T)$ be the PAA framework corresponding to $T$, for $\AAF^T=(\Args^T,\Att^T)$ and $\PS^T=(\Worlds^T,P)$. 
Then,  
the \emph{grounded probability of (ground) query} $q$ is   
$$Prob_{\GE}(q)=
 \sum_{w \in \Worlds^T: \arg \in \GE(\AAF^T_w), \arg=\argur{\asmset}{q}{\ruleset}} P(w).
 $$
\end{definition}
This notion is more fine-grained than 
that 
of grounded probability of arguments, reflecting the structured nature of our PAA framework
.
We conjecture (and leave to future work) that this notion is an instance of 
that of grounded probability of sentences 
in the Probabilistic ABA of \cite{PAA-PABA}
.

\section{Results}
\label{sec:results}
We show that the semantics of ProbLog can be captured in probabilistic argumentation, as follows.

\begin{proposition}
Let $T$ be a (ground) ProbLog program and 
$(\AAF^T, \PS^T,  \applies^T)$ be the PAA framework corresponding to $T$, for $\AAF^T=(\Args^T,\Att^T)$ and $\PS^T=(\Worlds^T,P)$. 
Then,  for any (ground) query $q$:  

\begin{eqnarray}
 && P_s(q|T)=Prob_{\GE}(q);
 \\
 && P_s(q|T) \leq 
\sum_{\arg \in \Args^T:
 \arg=\argur{\asmset}{q}{\ruleset}} Prob_\GE(\arg).
 \end{eqnarray}
\end{proposition}

\begin{proof}
 (Sketch) $P_s(q|T)=\sum_{L=F'\cup R,  
F'\subseteq \{l\leftarrow | p::l\leftarrow \in  F\},L\models q} P(L|T)$ by definition. Given the correspondence between well-founded model of a LP and grounded extension of the AA framework for the ABA framework corresponding to the LP \cite{ABA}, and by definition of the probability space in the PAA framework corresponding to $T$, we obtain
$P_s(q|T)=\sum_{w \in \Worlds^T: \arg \in \GE(\AAF^T_w), \arg=\argur{\asmset}{q}{\ruleset}} P(w)$, proving  
(1). 
(1) implies (2) because
$\sum_{\arg \in \Args^T:
 \arg=\argur{\asmset}{q}{\ruleset}} Prob_\GE(\arg)
 = \sum_{\arg \in \Args^T:
 \arg=\argur{\asmset}{q}{\ruleset}}
 \sum_{w \in \Worlds: \arg \in \XE(\AAF_w) } P(w) 
 =  \sum_{w \in \Worlds} P(w) 
 \cdot |\{ \arg \in \Args^T \mid
 \arg=\argur{\asmset}{q}{\ruleset}\}|
 \geq \sum_{w \in \Worlds^T: \arg \in \GE(\AAF^T_w), \arg=\argur{\asmset}{q}{\ruleset}} P(w)
 $
(there may be multiple arguments for the same claim in the grounded extensions obtained by different choices of worlds).  
\end{proof}
Thus, the notion of  grounded probability of queries in the ProbLog-ABA instance of PAA 
corresponds exactly  to the notion of success probability in ProbLog, whereas grounded probability of arguments
in PAA 
approximates it.


\section{Conclusion and Discussion}
\label{sec:concl}

We have defined a form of 
PAA,
instantiated with a novel form of ABA and probability spaces mirroring ProbLog's probability distributions over logic programs, to re-interpret ProbLog in argumentative terms.  
This allows us to broaden the semantics of ProbLog beyond (two-valued) well-founded models, leveraging on other semantics for ABA, notably semantics of \emph{sceptically preferred} and \emph{ideal extensions} (see \cite{ABAhandbook}
). 
Moreover, it opens the way to different forms of explainability for the outputs of ProbLog (under the standard or new semantics). Indeed, argumentative abstractions of various reasoning problems have been shown to lend themselves to 
diverse explanatory formats, including interactive ones~\cite{argXAI}, and it is well known that   
different forms of explanations are 
needed
to deal with different cognitive needs 
\cite{IBM}
.

Future work includes exploring whether reasoning under the new semantics can be efficiently implemented in practice, and whether the new forms of explanations can be beneficially deployed in applications to increase user trust. 
Also, we have focused on success probability of queries: it would also be interesting to study  whether and how
 explanation probability of queries can be captured argumentatively
 .
 It would also be interesting to explore whether 
ProbLog can be captured directly in Probabilistic ABA  
of the more general form proposed in \cite{PAA-PABA} (again focusing on a single juror)
.
It would also be interesting to explore whether credulous versions of ProbLog, notably \textsc{sm}ProbLog~\cite{SMProbLog}, could inform the definition of novel  credulous semantics in probabilistic argumentation. 
Finally,
it would be interesting to study whether other  languages/tools based on the distribution semantics~\cite{distribution} could also be captured 
in PAA.

\vspace*{-0.2cm} 
\section*{Acknowledgements} 
\vspace*{-0.2cm} 
Toni and Potyka were supported by the 
ERC under
the 
EU’s Horizon 2020 research and innovation programme (grant
No. 101020934). Toni was also supported
by 
J.P. Morgan and  the
UK RAEng under the Research Chairs
and Senior Research Fellowships scheme.
Ulbricht was supported by the German Federal Ministry of Education and Research (
01/S18026A-F) by funding 
Big Data and AI ``ScaDS.AI'' Dresden/Leipzig.
Totis was supported by the FWO project N. G066818N and the Flanders AI program. 

\bibliographystyle{eptcs}
\bibliography{bib}

\begin{thebibliography}{10}
\providecommand{\bibitemdeclare}[2]{}
\providecommand{\surnamestart}{}
\providecommand{\surnameend}{}
\providecommand{\urlprefix}{Available at }
\providecommand{\url}[1]{\texttt{#1}}
\providecommand{\href}[2]{\texttt{#2}}
\providecommand{\urlalt}[2]{\href{#1}{#2}}
\providecommand{\doi}[1]{doi:\urlalt{https://doi.org/#1}{#1}}
\providecommand{\eprint}[1]{arXiv:\urlalt{https://arxiv.org/abs/#1}{#1}}
\providecommand{\bibinfo}[2]{#2}

\bibitemdeclare{article}{ObjTraciking}
\bibitem{ObjTraciking}
\bibinfo{author}{Laura \surnamestart Antanas\surnameend},
  \bibinfo{author}{Plinio \surnamestart Moreno\surnameend},
  \bibinfo{author}{Marion \surnamestart Neumann\surnameend},
  \bibinfo{author}{Rui~Pimentel \surnamestart de~Figueiredo\surnameend},
  \bibinfo{author}{Kristian \surnamestart Kersting\surnameend},
  \bibinfo{author}{Jos{\'{e}} \surnamestart Santos{-}Victor\surnameend} \&
  \bibinfo{author}{Luc~De \surnamestart Raedt\surnameend}
  (\bibinfo{year}{2019}): \emph{\bibinfo{title}{Semantic and geometric
  reasoning for robotic grasping: a probabilistic logic approach}}.
\newblock {\slshape \bibinfo{journal}{Auton. Robots}}
  \bibinfo{volume}{43}(\bibinfo{number}{6}), pp. \bibinfo{pages}{1393--1418},
  \doi{10.1007/s10514-018-9784-8}.

\bibitemdeclare{article}{IBM}
\bibitem{IBM}
\bibinfo{author}{Vijay \surnamestart Arya\surnameend}, \bibinfo{author}{Rachel
  K.~E. \surnamestart Bellamy\surnameend}, \bibinfo{author}{Pin{-}Yu
  \surnamestart Chen\surnameend}, \bibinfo{author}{Amit \surnamestart
  Dhurandhar\surnameend}, \bibinfo{author}{Michael \surnamestart
  Hind\surnameend}, \bibinfo{author}{Samuel~C. \surnamestart
  Hoffman\surnameend}, \bibinfo{author}{Stephanie \surnamestart
  Houde\surnameend}, \bibinfo{author}{Q.~Vera \surnamestart Liao\surnameend},
  \bibinfo{author}{Ronny \surnamestart Luss\surnameend},
  \bibinfo{author}{Aleksandra \surnamestart Mojsilovic\surnameend},
  \bibinfo{author}{Sami \surnamestart Mourad\surnameend},
  \bibinfo{author}{Pablo \surnamestart Pedemonte\surnameend},
  \bibinfo{author}{Ramya \surnamestart Raghavendra\surnameend},
  \bibinfo{author}{John~T. \surnamestart Richards\surnameend},
  \bibinfo{author}{Prasanna \surnamestart Sattigeri\surnameend},
  \bibinfo{author}{Karthikeyan \surnamestart Shanmugam\surnameend},
  \bibinfo{author}{Moninder \surnamestart Singh\surnameend},
  \bibinfo{author}{Kush~R. \surnamestart Varshney\surnameend},
  \bibinfo{author}{Dennis \surnamestart Wei\surnameend} \&
  \bibinfo{author}{Yunfeng \surnamestart Zhang\surnameend}
  (\bibinfo{year}{2020}): \emph{\bibinfo{title}{{AI} Explainability 360: An
  Extensible Toolkit for Understanding Data and Machine Learning Models}}.
\newblock {\slshape \bibinfo{journal}{J. Mach. Learn. Res.}}
  \bibinfo{volume}{21}, pp. \bibinfo{pages}{130:1--130:6}.
\newblock \urlprefix\url{http://jmlr.org/papers/v21/19-1035.html}.

\bibitemdeclare{article}{ABA}
\bibitem{ABA}
\bibinfo{author}{Andrei \surnamestart Bondarenko\surnameend},
  \bibinfo{author}{Phan~Minh \surnamestart Dung\surnameend},
  \bibinfo{author}{Robert~A. \surnamestart Kowalski\surnameend} \&
  \bibinfo{author}{Francesca \surnamestart Toni\surnameend}
  (\bibinfo{year}{1997}): \emph{\bibinfo{title}{An Abstract,
  Argumentation-Theoretic Approach to Default Reasoning}}.
\newblock {\slshape \bibinfo{journal}{Artif. Intell.}} \bibinfo{volume}{93},
  pp. \bibinfo{pages}{63--101}, \doi{10.1016/S0004-3702(97)00015-5}.

\bibitemdeclare{article}{ABAhandbook}
\bibitem{ABAhandbook}
\bibinfo{author}{Kristijonas \surnamestart Cyras\surnameend},
  \bibinfo{author}{Xiuyi \surnamestart Fan\surnameend},
  \bibinfo{author}{Claudia \surnamestart Schulz\surnameend} \&
  \bibinfo{author}{Francesca \surnamestart Toni\surnameend}
  (\bibinfo{year}{2017}): \emph{\bibinfo{title}{Assumption-based Argumentation:
  Disputes, Explanations, Preferences}}.
\newblock {\slshape \bibinfo{journal}{{FLAP}}}
  \bibinfo{volume}{4}(\bibinfo{number}{8}).
\newblock
  \urlprefix\url{http://www.collegepublications.co.uk/downloads/ifcolog00017.pdf}.

\bibitemdeclare{article}{Quentin}
\bibitem{Quentin}
\bibinfo{author}{Kristijonas \surnamestart Cyras\surnameend},
  \bibinfo{author}{Quentin \surnamestart Heinrich\surnameend} \&
  \bibinfo{author}{Francesca \surnamestart Toni\surnameend}
  (\bibinfo{year}{2021}): \emph{\bibinfo{title}{Computational complexity of
  flat and generic Assumption-Based Argumentation, with and without
  probabilities}}.
\newblock {\slshape \bibinfo{journal}{Artif. Intell.}} \bibinfo{volume}{293},
  \doi{10.1016/j.artint.2020.103449}.

\bibitemdeclare{inproceedings}{argXAI}
\bibitem{argXAI}
\bibinfo{author}{Kristijonas \surnamestart Cyras\surnameend},
  \bibinfo{author}{Antonio \surnamestart Rago\surnameend},
  \bibinfo{author}{Emanuele \surnamestart Albini\surnameend},
  \bibinfo{author}{Pietro \surnamestart Baroni\surnameend} \&
  \bibinfo{author}{Francesca \surnamestart Toni\surnameend}
  (\bibinfo{year}{2021}): \emph{\bibinfo{title}{Argumentative {XAI:} {A}
  Survey}}.
\newblock In: {\slshape \bibinfo{booktitle}{Proceedings of the Thirtieth
  International Joint Conference on Artificial Intelligence, {IJCAI} 2021}},
  pp. \bibinfo{pages}{4392--4399}, \doi{10.24963/ijcai.2021/600}.

\bibitemdeclare{inproceedings}{synth}
\bibitem{synth}
\bibinfo{author}{Yann \surnamestart Dauxais\surnameend},
  \bibinfo{author}{Cl{\'{e}}ment \surnamestart Gautrais\surnameend},
  \bibinfo{author}{Anton \surnamestart Dries\surnameend},
  \bibinfo{author}{Arcchit \surnamestart Jain\surnameend},
  \bibinfo{author}{Samuel \surnamestart Kolb\surnameend},
  \bibinfo{author}{Mohit \surnamestart Kumar\surnameend},
  \bibinfo{author}{Stefano \surnamestart Teso\surnameend},
  \bibinfo{author}{Elia~Van \surnamestart Wolputte\surnameend},
  \bibinfo{author}{Gust \surnamestart Verbruggen\surnameend} \&
  \bibinfo{author}{Luc~De \surnamestart Raedt\surnameend}
  (\bibinfo{year}{2019}): \emph{\bibinfo{title}{SynthLog: {A} Language for
  Synthesising Inductive Data Models (Extended Abstract)}}.
\newblock In: {\slshape \bibinfo{booktitle}{Machine Learning and Knowledge
  Discovery in Databases - International Workshops of {ECML} {PKDD} 2019, Part
  {I}}}, \doi{10.1007/978-3-030-43823-4\_9}.

\bibitemdeclare{article}{Dung95}
\bibitem{Dung95}
\bibinfo{author}{Phan~Minh \surnamestart Dung\surnameend}
  (\bibinfo{year}{1995}): \emph{\bibinfo{title}{{On the Acceptability of
  Arguments and its Fundamental Role in Nonmonotonic Reasoning, Logic
  Programming and n-Person Games}}}.
\newblock {\slshape \bibinfo{journal}{Artif. Intell.}}
  \bibinfo{volume}{77}(\bibinfo{number}{2}), pp. \bibinfo{pages}{321--358},
  \doi{10.1016/0004-3702(94)00041-X}.

\bibitemdeclare{inproceedings}{PAA-PABA}
\bibitem{PAA-PABA}
\bibinfo{author}{Phan~Minh \surnamestart Dung\surnameend} \&
  \bibinfo{author}{Phan~Minh \surnamestart Thang\surnameend}
  (\bibinfo{year}{2010}): \emph{\bibinfo{title}{Towards (Probabilistic)
  Argumentation for Jury-based Dispute Resolution}}.
\newblock In: {\slshape \bibinfo{booktitle}{Computational Models of Argument:
  Proceedings of {COMMA} 2010}}, {\slshape \bibinfo{series}{Frontiers in
  Artificial Intelligence and Applications}} \bibinfo{volume}{216},
  \bibinfo{publisher}{{IOS} Press}, pp. \bibinfo{pages}{171--182},
  \doi{10.3233/978-1-60750-619-5-171}.

\bibitemdeclare{article}{ProbLog}
\bibitem{ProbLog}
\bibinfo{author}{Daan \surnamestart Fierens\surnameend},
  \bibinfo{author}{Guy~Van \surnamestart den Broeck\surnameend},
  \bibinfo{author}{Joris \surnamestart Renkens\surnameend},
  \bibinfo{author}{Dimitar~Sht. \surnamestart Shterionov\surnameend},
  \bibinfo{author}{Bernd \surnamestart Gutmann\surnameend},
  \bibinfo{author}{Ingo \surnamestart Thon\surnameend}, \bibinfo{author}{Gerda
  \surnamestart Janssens\surnameend} \& \bibinfo{author}{Luc~De \surnamestart
  Raedt\surnameend} (\bibinfo{year}{2015}): \emph{\bibinfo{title}{Inference and
  learning in probabilistic logic programs using weighted Boolean formulas}}.
\newblock {\slshape \bibinfo{journal}{Theory Pract. Log. Program.}}
  \bibinfo{volume}{15}(\bibinfo{number}{3}), pp. \bibinfo{pages}{358--401},
  \doi{10.1017/S1471068414000076}.

\bibitemdeclare{article}{WFM}
\bibitem{WFM}
\bibinfo{author}{Allen~Van \surnamestart Gelder\surnameend},
  \bibinfo{author}{Kenneth~A. \surnamestart Ross\surnameend} \&
  \bibinfo{author}{John~S. \surnamestart Schlipf\surnameend}
  (\bibinfo{year}{1991}): \emph{\bibinfo{title}{The Well-Founded Semantics for
  General Logic Programs}}.
\newblock {\slshape \bibinfo{journal}{J. {ACM}}}
  \bibinfo{volume}{38}(\bibinfo{number}{3}), pp. \bibinfo{pages}{620--650},
  \doi{10.1145/116825.116838}.

\bibitemdeclare{inproceedings}{SMS}
\bibitem{SMS}
\bibinfo{author}{Michael \surnamestart Gelfond\surnameend} \&
  \bibinfo{author}{Vladimir \surnamestart Lifschitz\surnameend}
  (\bibinfo{year}{1988}): \emph{\bibinfo{title}{The Stable Model Semantics for
  Logic Programming}}.
\newblock In: {\slshape \bibinfo{booktitle}{Logic Programming, Proceedings of
  the Fifth International Conference and Symposium,}}, pp.
  \bibinfo{pages}{1070--1080}.

\bibitemdeclare{inproceedings}{prism3val}
\bibitem{prism3val}
\bibinfo{author}{Spyros \surnamestart Hadjichristodoulou\surnameend} \&
  \bibinfo{author}{David~Scott \surnamestart Warren\surnameend}
  (\bibinfo{year}{2012}): \emph{\bibinfo{title}{Probabilistic Logic Programming
  with Well-Founded Negation}}.
\newblock In \bibinfo{editor}{D.~Michael \surnamestart Miller\surnameend} \&
  \bibinfo{editor}{Vincent~C. \surnamestart Gaudet\surnameend}, editors:
  {\slshape \bibinfo{booktitle}{42nd {IEEE} International Symposium on
  Multiple-Valued Logic, {ISMVL} 2012, Victoria, BC, Canada, May 14-16, 2012}},
  \bibinfo{publisher}{{IEEE} Computer Society}, pp. \bibinfo{pages}{232--237},
  \doi{10.1109/ISMVL.2012.26}.

\bibitemdeclare{article}{hunter2021probabilistic}
\bibitem{hunter2021probabilistic}
\bibinfo{author}{Anthony \surnamestart Hunter\surnameend},
  \bibinfo{author}{Sylwia \surnamestart Polberg\surnameend},
  \bibinfo{author}{Nico \surnamestart Potyka\surnameend},
  \bibinfo{author}{Tjitze \surnamestart Rienstra\surnameend} \&
  \bibinfo{author}{Matthias \surnamestart Thimm\surnameend}
  (\bibinfo{year}{2021}): \emph{\bibinfo{title}{Probabilistic argumentation: A
  survey}}.
\newblock {\slshape \bibinfo{journal}{Handbook of Formal Argumentation}}
  \bibinfo{volume}{2}, pp. \bibinfo{pages}{397--441}.

\bibitemdeclare{incollection}{abduction}
\bibitem{abduction}
\bibinfo{author}{Antonis~C. \surnamestart Kakas\surnameend},
  \bibinfo{author}{Robert~A. \surnamestart Kowalski\surnameend} \&
  \bibinfo{author}{Francesca \surnamestart Toni\surnameend}
  (\bibinfo{year}{1998}): \emph{\bibinfo{title}{The role of abduction in logic
  programming}}.
\newblock In: {\slshape \bibinfo{booktitle}{Handbook of logic in artificial
  intelligence and logic programming}}, pp. \bibinfo{pages}{35--324}.

\bibitemdeclare{incollection}{MetaNetworks}
\bibitem{MetaNetworks}
\bibinfo{author}{Angelika \surnamestart Kimmig\surnameend} \&
  \bibinfo{author}{Fabrizio \surnamestart Costa\surnameend}
  (\bibinfo{year}{2012}): \emph{\bibinfo{title}{Link and Node Prediction in
  Metabolic Networks with Probabilistic Logic}}.
\newblock In: {\slshape \bibinfo{booktitle}{Bisociative Knowledge Discovery -
  An Introduction to Concept, Algorithms, Tools, and Applications}}, pp.
  \bibinfo{pages}{407--426}, \doi{10.1007/978-3-642-31830-6\_29}.

\bibitemdeclare{inproceedings}{nir}
\bibitem{nir}
\bibinfo{author}{Hengfei \surnamestart Li\surnameend}, \bibinfo{author}{Nir
  \surnamestart Oren\surnameend} \& \bibinfo{author}{Timothy~J. \surnamestart
  Norman\surnameend} (\bibinfo{year}{2011}):
  \emph{\bibinfo{title}{Probabilistic Argumentation Frameworks}}.
\newblock In: {\slshape \bibinfo{booktitle}{Theories and Applications of Formal
  Argumentation - First International Workshop, {TAFA} 2011. Revised Selected
  Papers}}, {\slshape \bibinfo{series}{Lecture Notes in Computer Science}}
  \bibinfo{volume}{7132}, \bibinfo{publisher}{Springer}, pp.
  \bibinfo{pages}{1--16}, \doi{10.1007/978-3-642-29184-5\_1}.

\bibitemdeclare{inproceedings}{deepProbLog18}
\bibitem{deepProbLog18}
\bibinfo{author}{Robin \surnamestart Manhaeve\surnameend},
  \bibinfo{author}{Sebastijan \surnamestart Dumancic\surnameend},
  \bibinfo{author}{Angelika \surnamestart Kimmig\surnameend},
  \bibinfo{author}{Thomas \surnamestart Demeester\surnameend} \&
  \bibinfo{author}{Luc~De \surnamestart Raedt\surnameend}
  (\bibinfo{year}{2018}): \emph{\bibinfo{title}{DeepProbLog: Neural
  Probabilistic Logic Programming}}.
\newblock In: {\slshape \bibinfo{booktitle}{Advances in Neural Information
  Processing Systems 31: NeurIPS 2018,}}, pp. \bibinfo{pages}{3753--3763}.
\newblock
  \urlprefix\url{https://proceedings.neurips.cc/paper/2018/hash/dc5d637ed5e62c36ecb73b654b05ba2a-Abstract.html}.

\bibitemdeclare{article}{Bistarelli20}
\bibitem{Bistarelli20}
\bibinfo{author}{Theofrastos \surnamestart Mantadelis\surnameend} \&
  \bibinfo{author}{Stefano \surnamestart Bistarelli\surnameend}
  (\bibinfo{year}{2020}): \emph{\bibinfo{title}{Probabilistic abstract
  argumentation frameworks, a possible world view}}.
\newblock {\slshape \bibinfo{journal}{Int. J. Approx. Reason.}}
  \bibinfo{volume}{119}, pp. \bibinfo{pages}{204--219},
  \doi{10.1016/j.ijar.2019.12.006}.

\bibitemdeclare{article}{ProbLog15}
\bibitem{ProbLog15}
\bibinfo{author}{Luc~De \surnamestart Raedt\surnameend} \&
  \bibinfo{author}{Angelika \surnamestart Kimmig\surnameend}
  (\bibinfo{year}{2015}): \emph{\bibinfo{title}{Probabilistic (logic)
  programming concepts}}.
\newblock {\slshape \bibinfo{journal}{Mach. Learn.}}
  \bibinfo{volume}{100}(\bibinfo{number}{1}), pp. \bibinfo{pages}{5--47},
  \doi{10.1007/s10994-015-5494-z}.

\bibitemdeclare{inproceedings}{distribution}
\bibitem{distribution}
\bibinfo{author}{Taisuke \surnamestart Sato\surnameend} (\bibinfo{year}{1995}):
  \emph{\bibinfo{title}{A Statistical Learning Method for Logic Programs with
  Distribution Semantics}}.
\newblock In: {\slshape \bibinfo{booktitle}{Proceedings of the Twelfth
  International Conference on Logic Programming}}, pp.
  \bibinfo{pages}{715--729}.

\bibitemdeclare{article}{ABAtutorial}
\bibitem{ABAtutorial}
\bibinfo{author}{Francesca \surnamestart Toni\surnameend}
  (\bibinfo{year}{2014}): \emph{\bibinfo{title}{A tutorial on Assumption-based
  Argumentation}}.
\newblock {\slshape \bibinfo{journal}{Arg. \& Comp}}
  \bibinfo{volume}{5}(\bibinfo{number}{1}), pp. \bibinfo{pages}{89--117}.

\bibitemdeclare{article}{SMProbLog}
\bibitem{SMProbLog}
\bibinfo{author}{Pietro \surnamestart Totis\surnameend},
  \bibinfo{author}{Angelika \surnamestart Kimmig\surnameend} \&
  \bibinfo{author}{Luc~De \surnamestart Raedt\surnameend}
  (\bibinfo{year}{2023}): \emph{\bibinfo{title}{smProbLog: Stable Model
  Semantics in ProbLog for Probabilistic Argumentation}}.
\newblock {\slshape \bibinfo{journal}{Theory and Practice of Logic
  Programming}}, p. \bibinfo{pages}{1–50}, \doi{10.1017/S147106842300008X}.

\end{thebibliography}
\end{document}